\documentclass[final, 12pt]{colt2018arxiv} 

\title[Unleashing Linear Optimizers for Group-Fair Learning and Optimization]{Unleashing Linear Optimizers for Group-Fair Learning and Optimization}

\usepackage{times}

\coltauthor{\Name{Daniel Alabi} \Email{alabid@g.harvard.edu}\\
  \addr Harvard University, Cambridge, MA
  \AND
  \Name{Nicole Immorlica} \Email{nicimm@microsoft.com}\\
  \Name{Adam Tauman Kalai} \Email{adam.kalai@microsoft.com}\\
  \addr Microsoft Research New England, Cambridge, MA
}

\usepackage{url}            
\usepackage{amsmath}
\usepackage[colorinlistoftodos]{todonotes}

\usepackage{wrapfig}
\usepackage{caption}
\graphicspath{ {images/} }

\usepackage{booktabs}       
\usepackage{amsfonts}       
\usepackage{nicefrac}       
\usepackage{microtype}      

\usepackage{amsfonts}

\usepackage{algorithm}
\usepackage{algpseudocode}
\usepackage{caption}
\usepackage{comment}
\usepackage{dsfont}
\usepackage{graphicx,verbatim}
\usepackage{listings}

\def\reals{{\mathbb R}}

\def\calW{\mathcal{W}}
\newcommand{\eps}{\epsilon}
\newcommand{\E}{\mathop{\mathbb{E}}}
\newcommand{\err}{\mathrm{err}}
\newcommand{\pmin}{p_{\text{min}}}

\def\calC{\mathcal{C}}

\def\calS{\mathcal{S}} 

\def\calX{\mathcal{X}}

\def\calZ{\mathcal{Z}}
\def\ERR{\mathrm{ERR}}
\def\PRE{\mathrm{PRE}}
\def\REC{\mathrm{REC}}
\def\FONE{\mathrm{F1}}
\def\FP{\mathrm{FP}}

\def\FPR{\mathrm{FPR}}
\def\FNR{\mathrm{FNR}}

\def\norm#1{\mathopen\| #1 \mathclose\|}
\newcommand{\poly}{\mathop{\mbox{\text{poly}}}}

\newcommand{\ind}{\mathds{1}}
\newtheorem{observation}[theorem]{Observation}

\definecolor{orange}{rgb}{1,0.5,0}

\usepackage{comment}
\usepackage{dsfont}
\usepackage{todonotes}

\definecolor{navy}{rgb}{0,0,0.5}

\begin{document}

\maketitle

\begin{abstract}
Most systems and learning algorithms optimize average performance or average loss -- one reason being computational complexity. However, many objectives of practical interest are more complex than simply average loss. This arises, for example, when balancing performance or loss with fairness across people. We prove that, from a computational perspective, optimizing arbitrary objectives that take into account performance over a small number of groups is not significantly harder to optimize than average performance. Our main result is a polynomial-time reduction that uses a linear optimizer to optimize an arbitrary (Lipschitz continuous) function of performance over a (constant) number of possibly-overlapping groups. This includes fairness objectives over small numbers of groups, and we further point out that other existing notions of fairness such as individual fairness can be cast as convex optimization and hence more standard convex techniques can be used. Beyond learning, our approach applies to multi-objective optimization, more generally.
\end{abstract}


\section{Introduction}
The objectives of today's algorithms may be complex and non-linear. For example, a 2016 Bloomberg report demonstrated systematic bias in the roll-out of a major online retailer's delivery service~\citep{Ingold16}. It argued that, in many major metropolitan areas like Chicago, Boston, and New York, the service area excluded predominantly black zip codes. We give reductions that might help such decision-makers account for complex factors such as fairness, from choosing facility locations, to advertising, to product selection and pricing. We show how existing systems  that optimize ``average performance'' can be used to optimize more complex objective functions. Our work builds on a line of linear reductions in algorithms and fair machine learning \citep[e.g.,][]{revisitingFrankWolfe, KalaiV05,NarasimhanRS015,Agarwal17,dwork18a}.

Consider the following optimization problem:
\begin{equation}\label{eq:duh}
\min_{c \in \calC} f(\ell_1(c), \ldots, \ell_K(c))\end{equation}
Here $\calC$ is a (possibly constrained) set of choices and \textit{loss vector} $\ell(c)\in [0,1]^K$ represents costs or losses to $K$ possibly-overlapping groups, such as racial and gender groups. As a simple optimization example, consider the classic Facility Location Problem (FLOP), in which an algorithm is given a graph $G$ and customer weights, and must select $m$ locations for facilities amongst the nodes on a directed graph, so as to minimize the weighted average distance of nodes to the closest facility. FLOPs model a firm choosing a limited number of locations to serve a set of customers, where $c$ could encode the $m$ chosen facilities and $\ell_k(c)$ could encode the average distance over customers in that group. The customer weights can also be used to reweight groups, hence a standard FLOP algorithm can minimize $w \cdot \ell(c)$ for any given group weight vector $w$. But natural objectives $f$ may be complex and non-convex, even without factoring in fairness. For instance, in a FLOP the marginal advantage of reducing group $k$'s average delivery time $l_k=\ell_k(c)$ from 2 hours to 1 hour may be more significant than reducing it from 48 hours to 47 hours, as would be captured by a concave function such as $f(l) = \sum_k \sqrt{l_k}$. Convexity is often required for algorithmic convenience, but may not be natural in any given application. 

As an ML example, consider a loan approval decision where the choices are approval criteria, i.e., classifiers $c(x)$. While it may be illegal for the approval decisions to depend on sex, the choice of $c$ may take into account resulting demographic differences. 
For simplicity, say there are two sexes and four groups based on the individual's sex (M or F) and whether $y=1$ indicating they would repay or $y=0$ indicating they would default. Let the loss vector $\ell(c)\in [0,1]^4$ consist of the misclassification rates on the four groups, which are the F/M false-positive/negative rates. Each classifier corresponds to a feasible point in $[0,1]^4$. For the moment, consider only classifiers with no false positives (no approvals that default), so the feasible set $F=\{ \ell(c)|c \in \calC\}$ can be visualized in the two dimensions of F/M false negative rates. The feasible set is the set of achievable F- and M-false-negative rates achievable by classifiers, e.g., each neural network corresponds to a feasible point. Since one may randomize amongst classifiers, one can further assume that the feasible set is convex.

Unfortunately the feasible set $F$ is not directly accessible, e.g., we do not know if there is a neural network that achieves specified F/M false negative rates (or specified average group distances in FLOP). \cite{NarasimhanRS015} showed how to use binary classification algorithms to optimize in any linear direction and consequently closely approximate $F$. In particular, they showed how to use a Frank-Wolfe style reduction to minimize convex objectives as well as ratio-of-linear objectives. \cite{Agarwal17} provide a related reduction for fair ML with arbitrarily many groups and convex constraints. In this work, we consider the case of small numbers of groups, $K$, where we show we can optimize arbitrary (Lipschitz) continuous objectives.

For example, consider the following stylized objective that combines misclassification rate with a penalty for M-F disparity amongst approvals:
\begin{equation}\label{eq:example-obj}
\min_{\stackrel{c \in \calC:}{\Pr[c(x)=1\wedge y=0] =0}} \Pr[c(x)\neq y] + \left(\Pr\left[x_{\text{sex}}=\text{F} ~\bigl|~c(x)=1\right]-\Pr\left[x_{\text{sex}}=\text{M} ~\bigl|~c(x)=1\right]\right)^2
\end{equation}
This objective (and the false positive restriction) can be computed in terms of the loss vector $\ell(c)$. However, as illustrated in Figure \ref{fig:example} (false positive rates not shown because they are all 0), this objective is not convex and can have local minima that are not global minima. 

Fairness is complicated by many factors such as biased training labels and societal biases that cannot be corrected by loan approval alone. In Figure \ref{fig:example}, exactly equal F/M rates are only achieved by rejecting all loans, which may not be very useful to a loan-granting organization. Legal definitions of discrimination discuss process and other factors beyond approval rates \citep{Barocas2016}. Nonetheless, given that algorithms are making such decisions, the flexibility to use existing algorithms to optimize complex non-linear objectives enables domain experts to choose objectives that may lead to desirable outcomes, accounting for biased data, fairness, and of course performance. Our reduction helps because it can reuse existing linear systems and can circumvent local optima that might trap more direct methods.

\begin{figure}
\centering
\includegraphics[width=5in]{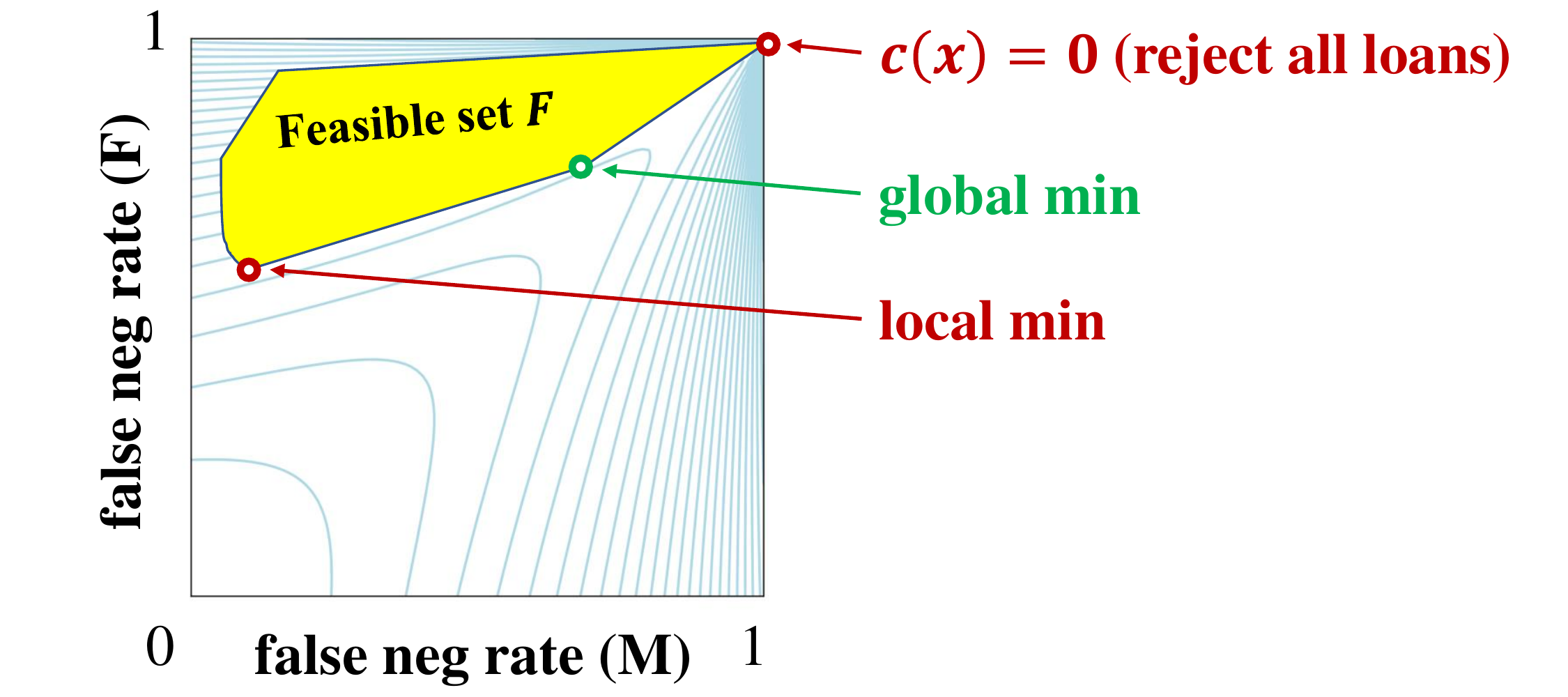}
\caption{A contour plot of non-convex objective $f$ from eq.~(\ref{eq:example-obj}) as a function of M/F false negative rates (with 0 false positives, and uniform 1/4 probability of $\pm$M/$\pm$F). The set of false positive/negative rates achieved by classifiers (yellow) has a local (non-global) minimum. 
\label{fig:example}}
\end{figure}

Algorithm \ref{alg:groupopt} gives a (polynomial-time) reduction that can minimize (Lipschitz) continuous functions $f$ (for constant $K$), using an underlying algorithm that can minimize linear loss $\sum w_k \ell_k(c)$, for arbitrary real-valued weights $w_k$. As mentioned, customer weights standard in FLOP algorithms (meant to indicate for instance the frequency of serving these customers) can be used to weight groups, and hence repeated calls to the FLOP algorithm on the same graph with different weights are used to efficiently minimize $f$. For binary classification, a corollary of our main theorem is the following:

\noindent
\textbf{Corollary \ref{cor:classification} (Informal version)} \textit{
Let $M$ be an agnostic learner for family $\calC$ of binary classifiers $c:\calX \rightarrow\{0,1\}$. Let $\calX_1, \ldots, \calX_K\subseteq \calX$ be sensitive groups. Then any Lipschitz continuous function of the error rates and false positive and negative rates across the groups with Lipschitz-continuous constraints can be minimized over $\calC$ to within $\eps$ in time and calls to $M$ that are $\poly(1/\eps$) for any constant $K$.}

\noindent
Agnostic learning \citep{Kearns:1992} is a natural model of computationally efficient binary classification. The work of \cite{NarasimhanRS015} and \cite{Agarwal17} shows how to optimize {\em convex} objectives (and ratio-of-linear objectives) with convex constraints on group error rates using an existing agnostic learner. The above corollary extends this to arbitrary continuous objectives, though it does require constant numbers of groups $K$. We also show (Corollary \ref{cor:regression}) how to use an agnostic learner to optimize a continuous and nondecreasing function of squared errors across groups -- hence a classifier can be used for a form of fair regression. 

NP-hard problem such as FLOP and many agnostic learning problems are intractable for some large instances, but algorithms may run quickly on ``real-world'' instances. For instance, as long as the FLOP algorithm runs quickly on a graph $G$, eq.~(\ref{eq:duh}) can be efficiently solved on $G$. Similarly, while neural network training algorithms are not guaranteed to find optimal solutions, in practice they often perform well. 


\smallskip \noindent
\textbf{Constraints}. Note that general (continuous) constraint functions $g(l)\leq 0$ can easily be folded into the objective --- if $R=\max_{l,l' \in F} f(l)- f(l')$ is the range of $f$ and $\mathrm{opt}$ is the minimum value of $f(l)$ over feasible $l=\ell(c)$ such that $g(l)\leq 0$ (assuming the constraint is satisfiable), then 
\begin{equation}\label{eq:constraints}
f(l) + \Gamma \max(0, g(l)) \leq \mathrm{opt} + \eps ~~\text{ implies }~~ f(l) \leq \mathrm{opt} + \eps \text{ and } g(l) \leq (\eps+R)/\Gamma.
\end{equation}
In our example (\ref{eq:example-obj}), constraining the false positive rate to $\leq \delta$ (it cannot be constrained to exactly 0 with finite samples) can be achieved by adding $\Gamma \Pr[c(x)=1\wedge y=0]$ to the objective, for $\Gamma \geq (R+\eps)/\delta$. A set of linear constraints $u_i \cdot l \leq b_i$ can be captured by $g(l)=\max_i u_i \cdot l - b_i$.

\smallskip \noindent
\textbf{Convex objectives.} If $f$ happens to be a convex function, then Section \ref{sec:convex} points out that approaches like that of \cite{NarasimhanRS015} can be used to convert a weighted-sum minimizer of $\sum w_i \ell_i(c)$ to efficiently optimize (\ref{eq:duh}) to within $\eps$ in time $\poly(1/\eps, K)$, which means that we can handle large numbers of groups or even $K=n$, i.e., a group for each customer. For instance, a FLOP algorithm could be used to find a (distribution over) $m$ facilities that minimize the maximum (expected) customer's distance to its nearest facility, or the Generalized Gini Index considered in fairness \citep{weymark1981, busa2017multi} and which has been studied in FLOP in particular \citep{drezner2009equitable}. It can also find classifier parameters for individual fairness (also called metric fairness) \citep{Dwork:2012}.

\smallskip \noindent
\textbf{Organization.} We next present related work followed by mathematical preliminaries. Section \ref{sec:group} presents our general reduction for continuous functions of losses across $K$ groups. Section \ref{sec:learning} shows how to apply this to learning. And Section \ref{sec:convex} shows how to reduce convex objectives to linear optimization.

\section{Related Work}

\textbf{Reductions in binary classification.}
\cite{NarasimhanRS015} show reductions to binary classification albeit for
performance measures that can be expressed as a ratio of linear functions or that
are concave. \cite{dwork18a} show how to reduce the problem of optimizing a group-fair loss (for disjoint groups) to standard binary classification, however their work requires that classifiers $c$ have access to the sensitive attribute determining group membership. Their approach is simply to learn a different classifier for each group, whereas in our work the classifiers may or may not (or may have partial) access to sensitive attributes. Their work further requires $f$ to be non-decreasing but, on the other hand, does not require $f$ to be continuous.
\cite{Agarwal17} show how to reduce convex fairness constraints, such as demographic parity and equalized odds, to binary classification.
 Within classification, our work builds upon and generalizes these works in that we can handle arbitrary continuous functions of
losses (for a constant number of possibly overlapping groups).

\medskip
\noindent
\textbf{Multi-objective optimization}.
When faced with multiple objectives, no one single solution may be optimal \citep[see, e.g.,][]{ZuluagaSKP13, RaviMRRH93, GrandoniRSZ14}. A solution is Pareto-optimal if no other solution is simultaneously better in all objectives. Often the goal is to recover the entire set of Pareto-optimal solutions. 
While computing the Pareto set (even membership) is NP-hard~\citep{RaviMRRH93}, 
for some problems it can be succinctly approximated \citep{warburton1987approximation, papadimitriou2000approximability, Diakonikolas:2011}. Some work in multi-objective optimization addresses the problem as a convex programming problem, while other work such as \cite{Diakonikolas:2011} provides reductions to an underlying linear optimization problem. Our work follows in this spirit.

\medskip
\noindent
\textbf{Optimizing non-decomposable ML losses}.
Misclassification rate is the most common performance measure studied in
learning theory.  However in cases of severe label imbalance, other evaluation metrics over the entire dataset may be used, but these metrics are often non-decomposable. For example, \cite{ParambathUG14} show how to reduce the problem of minimizing F-measures such as the F1 score, the harmonic mean of precision and recall, to linear optimization. Our work generalizes this work.

\medskip
\noindent
\textbf{Fairness in classification}.
The study of fairness in classification is an important emerging research area
~(\cite{Dwork:2012, JosephKMR16, Kearns:17}). Some impossibility results are
known. For instance, except in cases of equal base rates of true
classifications across groups or in very rare cases of perfect predictions,
there is no well-calibrated and class-balanced classifier
~(\cite{KleinbergMR17,Chouldechova17}). Despite such results,
we know that we can design multi-objective metrics that trade-off fairness for accuracy and
vice-versa~(\cite{Woodworth:2017, MenonW17, dwork18a, Goh16}). We show that optimizing fair multi-objectives
(for a constant number of possibly overlapping groups) is no harder than
optimizing approximate linear optimizers.
\cite{OgryczakLPNT14}
survey some models and methods for achieving
fairness, with a particular emphasis on the optimization
of networks. 

\medskip
\noindent
\textbf{Frank-Wolfe and online optimization}.
The general approach of reducing optimization to linear optimization has been utilized in many areas of machine learning \citep{revisitingFrankWolfe,KalaiMV08}. For the convex optimization, we use the reduction of \cite{KakadeKL09} as a sub-routine. We also leverage the fact that their reduction can handle minimizers with limited domains.

\section{Mathematical Notation}
Let $\Delta(S)$ denote the set of probability distributions over $S$ and let $T^S$ denote the set of functions $f: S \rightarrow T$, for sets $S,T$. For $U \subseteq S$, write $f(U) = \{f(u)~|~u \in U\}$. For $v, w\in \reals^K$, we write $v\leq w$ if $v_i\leq w_i$ for all $i$, and for scalar $r \in \reals$, the vector-scalar sum is $v+r=(v_1+r, \ldots, v_K+r)\in \reals^K$. A function $f: \reals^K \rightarrow \reals$ is said to be nondecreasing if it is nondecreasing in each coordinate, or equivalently $f(v) \leq f(w)$ for all $v\leq w$. We say that function $f:\reals^n \rightarrow \reals^d$ is convex if it is convex in each coordinate, i.e., if $f(\lambda v + (1-\lambda)w) \leq \lambda f(v) + (1-\lambda) f(w)$ for $\lambda \in [0,1]$.

For $v \in \reals^d$, let $(v)_+ = \bigl(\max(0, v_1), \ldots, \max(0, v_d)\bigr) \in \reals^d$ and  $\|v\|_\infty = \max_i v_i$. Let $\reals_+=[0,\infty)$. For set $S$, let $S^*=S \cup S \times S \cup \ldots$ be arbitrary-length sequences of elements of $S$.

Also, for $L \geq 0$,  $f: \reals^d \rightarrow \reals$ is said to be $L$-Lipschitz continuous (or just $L$-Lipschitz) if $|f(x)-f(y)| \leq L\|x-y\|$ for all $x,y \in \reals^d$. We say $f$ is Lipschitz if it is 1-Lipschitz. Finally, let $[n]=\{1,2,\ldots,n\}$ and $\ind[P]$ denote 1 if predicate $P$ holds, otherwise 0.

\section{Group Objectives}\label{sec:group}
Let $K$ be a constant number of groups, let $\calC$ be a set of choices, and let $\ell: \calC\rightarrow [0,1]^K$ be a loss function that determines the (bounded) average loss on each choice in $c$. We assume that $\calC$ is {\em closed under randomization} meaning that every $c_1, \ldots, c_n\in \calC$ can be identified with a choice $c \in \calC$, written $c=\mathrm{UniformDist}(\{c_1,\ldots,c_n\})$, with $\ell(c)=(\ell(c_1)+\cdots+\ell(c_n))/n$. This assumption is necessary for fairness, e.g., clearly randomization is necessary for fairness in locating a single facility on a two-node graph. Even for nonlinear $c$'s such as neural networks, it is possible to randomize between neural networks or average their outputs (which is very different than averaging the weights).  

We also fix a comparison set $\bar{\calC} \subseteq \calC$, which may be equal $\bar{\calC}=\calC$ but more generally allows an algorithm to output a choice in a different set $\calC$ but only be compared to choices from $\bar{\calC}$. This technicality allows us to capture so-called ``improper'' learning algorithms that learn a class $\bar{\calC}$ by outputting classifiers in a different family, e.g., $\bar{\calC}$ may be a class of bounded-size binary decision trees and $\calC$ may include the larger class of low-degree polynomials.  We also assume that $\bar\calC$ is closed under randomization, so our guarantees compared to the best $\bar{c}\in \bar\calC$ also compare to probability distributions over choices.

For $w \in \reals^K$ and $c \in \calC$, we call $w \cdot \ell(c)$ the weighted loss of $c$. We require that we can efficiently (approximately) minimize $w \cdot \ell(c)$. In some applications, this can only be done for non-negative $w \geq 0$, and for those we will have an additional monotonicity requirement on $f$. To this end, consider two types of linear optimizers: (a) a general linear optimizer $M_\tau:\reals^K\rightarrow \calC$ which can take arbitrary real weights; and (b) a nonnegative linear optimizer $M_\tau:\reals_+^K\rightarrow \calC$ , with its conditions in parentheses, which takes only nonnegative weights.

\begin{definition}[(Nonnegative) Linear optimizer]
\label{def:LOO}
A {\em (nonnegative) linear optimizer} $M_\tau$, for any tolerance $\tau>0$, for any $w \in \reals^K$ (with $w \geq 0$), outputs $M_\tau(w)\in \calC$ which, to within $\tau$, optimizes weighted loss, i.e.,
$$w \cdot \ell(M_\tau(w)) \leq \min_{\bar{c} \in \bar{\calC}} w \cdot \ell(\bar{c}) + \tau \|w\|$$
\end{definition}
The tolerance $\tau$ is scaled by $\|w\|$ due to the fact that a scaling of $w$ scales all $w \cdot \ell(c)$ by the same factor. Note that linear optimization does not need randomization: $\mathrm{UniformDist}(\{c_1,\ldots,c_n\})$ can always be replaced by the choice $c_i$ that minimizes $w \cdot \ell(c_i)$. For instance, if the family $\calC$ consists of probability distributions over neural networks, then one could simply choose the most accurate neural network without degrading performance. For FLOP, the linear minimizer will output a fixed set of locations. Our algorithm, however, will randomize over small number of neural networks or a small number of $t$-tuples of locations.  

Of course, we also must be able to compute the loss for each group, at least approximately:
\begin{definition}[Loss assessor]
\label{def:GLO}
A loss assessor $\ell_\tau$, for any tolerance $\tau>0$, computes a function $\ell_\tau: \calC \rightarrow [0,1]^K$ which approximates $\ell$ to within $\tau$, i.e., $\|\ell_\tau(c)-\ell(c)\|\leq \tau$ for all $c \in \calC$.
\end{definition}

Algorithm \ref{alg:groupopt} uses a (nonnegative) linear optimizer and a loss assessor
to minimize a (nondecreasing) Lipschitz continuous loss $f: [0,1]^K\rightarrow \reals$.  It works by creating a grid $G$ over $[0,1]^K$, fine enough that the optimum is within $O(\eps)$ of some grid point. (Note that the optimum's nearest grid point is nearly feasible.) We sort the nearly feasible grid points by $f(\cdot)$. For each such grid point $r$, in order from least to greatest $f(r)$, we see if we can find a $c\in \calC$ with $\ell(c)$ within $O(\eps)$ of $r$ by calling the minimizer $M(\ell(c)-r)$ and moving in that direction $T$ times.  Otherwise we move on to the next grid point. The algorithm is shown to stop on or before the grid point closest to the optimal feasible solution, and outputs a uniform distribution $c$ over a set of $T=O(1/\eps^2 \log 1/\eps^2)$ choice settings.

Several notes about the algorithm are in order. 
\begin{itemize}
\item While the algorithm outputs a convex combination of $T$ choices, this can always be reduced to a combination of at most $K+1$ choices (since we are in $[0,1]^K$) using standard Carathéodory's Theorem \citep[see, e.g.,][]{mulzer2014computational}.

\item Equation (\ref{eq:constraints}) shows how to fold constraints $g(x) \leq 0$ for (nondecreasing) Lipschitz $g:\reals^K \rightarrow \reals$ into the objective. 

\item While the algorithm requires Lipschitz (i.e, $L=1$) functions, it can be used for any $L$-Lipschitz $f$ for $L>1$ by simply applying the algorithm to Lipschitz function $f/L$ with accuracy $\eps/L$. The Lipschitz constant may depend on group demographics and on $\Gamma$ (\ref{eq:constraints}).
Further, Laplace smoothing \citep[see, e.g.][]{schutze2008introduction} can be used to make Lipschitz-continuous approximations to discontinuous objectives. For instance, the objective of eq.~\ref{eq:example-obj} is discontinuous at (1,1), but is continuous if one adds a constant to the counts of false positives/negatives in each group.

\end{itemize}

The theorem below provides guarantees for \textsc{GroupOpt} (with guarantees for nonnegative linear optimizer $M_\tau:\reals_+^K \rightarrow \calC$ written in parentheses). 

\begin{algorithm}
\caption{\textsc{GroupOpt}: Minimizing group-loss $f$ using linear optimizer}
\label{alg:groupopt}
\KwIn{accuracy $\eps>0$, $f,g:[0,1]^K\rightarrow\reals$, loss assessor $\ell_\tau$, (nonnegative) linear optimizer $M_\tau$}

\KwOut{$c \in \calC$}

\smallskip
Let $\beta = \frac{\eps}{5}$, $q = \frac{\beta}{\sqrt{K}}$, $\tau=\frac{\beta^2}{\sqrt{K}}$, $T = \frac{K}{\beta^2}\ln\frac{K}{\beta^2}$.

\smallskip
Create grid $G = \{0, q, 2q, 3q, \ldots, \lfloor 1/q\rfloor q\}^K\subseteq [0,1]^K$.

\smallskip
Check if $f$ is nondecreasing coordinatewise on $G$. If so, let $N=1$ else $N=0$.

\smallskip
Sort points in $G$ by $f(r)$ in increasing order

\smallskip
\For {$r$ in $G$}
{
   $c_1 = M_\tau(0)$ ~~~~// any initial choice

  \For {$t=1$ to $T$} {

      $\hat{l}_t = \max\left(\frac{1}{t}(\ell_\tau(c_1)+\cdots+\ell_\tau(c_t)), Nr\right)$

      $c_{t+1} = M_\tau(\hat{l}_t-r)$    
   }   
   \If{ $\left\|\hat{l}_{T}-r\right\|\leq 3\beta$} {
   \Return $c=\mathrm{UniformDist}(\{c_1, c_2, \ldots, c_T\})$  ~~~~// uniform probability distribution
   
 }
}
\end{algorithm}

\begin{theorem}\label{thm:group}
Let $K\geq 1$, $f: [0,1]^K \rightarrow \reals$ be (nondecreasing) Lipschitz-continuous functions, $\ell: \calC\rightarrow [0,1]^K$ be a group loss function with loss assessor $\ell_\tau$ and (nonnegative) linear optimizer $M_\tau$. 
For any $0<\eps\leq 1$,  $ f(\ell(\textsc{GroupOpt}(\eps, f, \ell_\tau, M_\tau)) \leq \min_{l \in \ell(\bar\calC)}f(l)+\eps$. For constant $K$, \textsc{GroupOpt} makes $\poly(1/\eps)$ queries to $\ell_\tau$ and $M_\tau$ both with $\tau=\frac{\eps^2}{25\sqrt{K}}$, and uses $\poly(1/\eps)$ additional runtime and calls to $f$.
\end{theorem}

Simply put, the computational complexity of optimizing continuous $f$ is the same, up to polynomial factors, as linear optimization, for a constant number of groups $K$. As we shall see, since $\ell$ corresponds to the loss on a specific instance, this theorem implies that for those graphs or distributions on unlabeled data $\calX$ for which (weighted) linear optimization is possible, group optimization is possible as well. The proof of Theorem \ref{thm:group}, which contains runtime bounds, is deferred to Appendix \ref{ap:proof:thm:group}. 

\section{Group Learning Applications}\label{sec:learning}
In this section, we show how Theorem \ref{thm:group} can be applied to binary classification and regression. In Section \ref{sec:regression-reduction}, we show  that regression can be reduced to classification. In Section \ref{sec:learning-meat}, we show that one can use a binary classifier to minimize arbitrary continuous group-loss objectives for classification and regression. 

We begin by defining a standard notion of computationally efficient learning called {\em Agnostic learning} \citep{Kearns:1992}. It is called agnostic because, unlike PAC learning \citep{valiant1984theory} which it generalizes, it makes no assumption on the distribution of data. 

\subsection{Agnostic Learning Definitions}\label{sec:learning-prelim}

Let $\calX$ be a set of feature vectors, $\calZ \subseteq \calX \times \{0,1\}$ be the set of examples, and $\mu\in \Delta(\calZ)$ be a distribution over $\calZ$. A {\em classifier} is a possibly-randomized function $c: \calX \rightarrow \{0,1\}$. Let $\calC\subseteq \{0,1\}^\calX$ be a set of \textit{legal} classifiers or predictors which can, for instance, enable restrictions on using protected attributes. Let $\bar{\calC} \subseteq\calC$ be a family of classifiers to be learned. Our framework is general and not restricted to the so-called \textit{proper} case when $\bar{\calC}=\calC$.  The error of a classifier $c$ is $$\err_\mu(c)=\Pr_{x,y \sim \mu}\left[y\neq c(x)\right].$$ 

To fit in our framework, we assume the classes $\calC$ and $\bar\calC$ are closed under randomization. For instance, if $\calC$ included rectangles in $\reals^d$, it would also include probability distributions over rectangles. However, optimization with respect to this larger class is no harder because the most accurate single classifier (e.g., rectangle) is always as accurate as the most accurate distribution over rectangles. Hence, our requirement of agnostic learning is weak -- one only needs to find a pure predictor, while our guarantees will be with respect to all distributions in $\bar\calC$, e.g., mixtures over rectangles, which is necessary for fairness.

\begin{definition}[Agnostic learning]
Polynomial-time algorithm $M:\calZ^*\rightarrow \calC$ is an agnostic learner for ${\bar{\calC}} \subseteq \calC \subseteq \{0,1\}^\calX$ ($[0,1]^\calX$) if there exists a polynomial $P(1/\eps, 1/\delta)$, such that for any $\eps, \delta>0, \mu \in \Delta(\calZ)$, with probability $\geq 1-\delta$ over $n\geq P(1/\eps, 1/\delta)$ examples $Z = \langle(x_1, y_1), \ldots (x_n, y_n)\rangle$ iid from $\mu$, $M$'s output has error within $\eps$ of optimal, i.e.,
$$\Pr_{Z \sim \mu^n}\left[\err_\mu({M(Z)}) \leq \min_{\bar{c} \in \bar{\calC}} \err_\mu({\bar{c}}) + \eps\right] \geq 1-\delta.$$
\end{definition}
Note that an agnostic learning algorithm is often defined as taking $\eps, \delta$ as input choices as well \citep{Kearns:1992}, however, one can simply think of the algorithm as taking an arbitrary number of examples $n$ as input and providing such a guarantee for a fixed schedule of $\eps, \delta$ that are inverse polynomial functions of $n$ (e.g. $\eps=\delta=n^{-1/4}$).

Distribution-specific agnostic learning refers to the case in which the above guarantee holds with respect to some specific distribution marginal on $x$. Common marginal distributions for which algorithms are analyzed include the uniform distribution on the hypercube $x\in\{0,1\}^d$ or any Gaussian or log-concave distribution.

\subsection{Group Agnostic Learning}\label{sec:learning-meat}

We now consider agnostic learning for $K$ groups $\calX_1, \ldots, \calX_K \subseteq \calX$, each with a group oracle, also denoted $\calX_k$, that efficiently computes group membership. Note that classifiers do not directly have access to group membership -- it is assumed to be available on training data but not for prediction on future examples. Typically, group membership is as simple as checking whether a protected attribute and/or label belong to a particular set. In our model, protected attributes are encoded in $\calX$, 
but the family $\calC$ of legal classifiers may not be allowed to depend on them. 

Let us illustrate on classification. Consider the following example for classification (with regression in parentheses).
\begin{example}
$\calX = \{-1,1\}^2 \times \reals^d$, where the first attribute $x_1$ encodes race, a binary protected attribute, meaning that it cannot be used for classification, while $x_2$ encodes gender, say also binary for illustration, which can be used for classification. The legal classifiers $\calC$ are all $c: \calX\rightarrow [0,1]$ such that $c(-1,x)=c(1,x)$ for any $x\in \{-1,1\} \times \reals^d$. Gender and race are needed for training data, while only gender is needed for classification. The family $\bar\calC$ being learned could be linear thresholds $c_\theta(x) = \ind[\theta \cdot x > 0]$ (with bounded norm $\|\theta\|$), with the additional requirement on $\theta \in \reals^{2+d}$ that $\theta_1=0$. 
\end{example}
To accommodate intersectionality \citep{crenshaw1989demarginalizing}, one can create groups for each combination of group attributes. For instance, in the race/gender example, one could create 8 groups, one for each combination of race, gender, and label. More generally, given a constant number of subsets of $\calX$, there are a constant number of intersections of groups and their complements together with binary labels. 
\begin{align*}
\FPR_k(c) &= \Pr[c(x)=1~|~x \in \calX_k \wedge y=0]\\ 
\FNR_k(c) &= \Pr[c(x)=0~|~x \in \calX_k \wedge y=1] 
\end{align*}
For each $k\in [K], i\in \{0,1\}$, define $p_k^i=\Pr[x \in \calX_k \wedge y=i]$ and 
$p_k=p_k^0+p_k^1=\mu(\calX_k\times\{0, 1\})$. For simplicity we assume that the algorithm is given exact knowledge of all $p_k^i$'s, but this assumption can be removed using standard estimation by Chernoff bounds. Note that with knowledge of $p_k^i$, the vector $(\FPR(c),\FNR(c))\in[0,1]^{2K}$ has sufficient statistics to compute all other quantities of interest such as the group error rates, precision, recall, and $\FONE$-score:
\begin{align}
\ERR_k(c) &= \Pr[c(x)\neq y ~|~x \in \calX_k] = \frac{p_k^0\FPR_k(c)  + p_k^1\FNR_k(c)}{p_k} \label{eq:group-err}\\
\PRE_k(c) &= 1-\FNR_k(c)\\
\REC_k(c) &= \frac{p_k^1(1-\FNR_k(c))}{p_k^1(1-\FNR_k(c)) + p_k^0\FPR_k(c)}\\
\FONE_k(c) &= \frac{2 \PRE_k(c)\REC_k(c)}{\PRE_k(c)+\REC_k(c)} \label{eq:f1}
\end{align}

The formal version of the corollary stated in the introduction is given below.  
\begin{corollary}\label{cor:classification}
Let $M$ be an efficient agnostic classification algorithm for $\bar{\calC}$. Let $L>0$ and $f:[0,1]^{2K} \rightarrow \reals$ be $L$-Lipschitz. For any $\eps, \delta>0$, with probability $\geq 1-\delta$, Algorithm \ref{alg:groupopt} together with oracles to $\calX_k$ and knowledge of $p_k^i=\mu(\calX_k \times \{i\})$ can be used to output $c \in \calC$  such that,
$$f(\FPR(c), \FNR(c))\leq \min_{\bar{c}\in \bar{\calC}} f(\FPR(\bar{c}), \FNR(\bar{c}))+\eps.$$
Let $\pmin = \min_{k,i} p_k^i$. For any constant $K$, the algorithm uses $\poly(L, 1/\eps, 1/\delta,  1/\pmin)$ total runtime, examples, calls to $f$, group membership oracles $\calZ_k$, and agnostic learner $M$. $M$ is called on examples with the same marginal distribution on $\calX$ as $\mu$. 
\end{corollary}
As seen from eq.~(\ref{eq:constraints}), arbitrary nonlinear continuous constraints $g(\FPR(c), \FNR(c))\leq \eps$ can be folded into the objective. The only subtlety in applying Algorithm \ref{alg:groupopt} to this problem is in minimizing weighted combinations of group false positive and false negative rates using an agnostic learner. However, a standard randomized label-flipping trick is used to effectively reweight examples, even with negative weights.  
 
\subsubsection{Regression Group Loss}\label{sec:regression-reduction}
In this section, we show how agnostic learning can be used to minimize a nondecreasing continuous function $f$ of squared error over groups.  

Define the set of finite linear combinations of classifiers with coefficient bound $B$ to be:
$$\calC_B=\left\{\left. \sum_i \alpha_i c_{i}~\right|~(\alpha_i, c_i)\in \reals\times \calC, \sum_i |\alpha_i|\leq B\right\}.$$
We similarly define the comparison class $\bar{\calC}_B\subseteq \calC_B$ to the analogous set of bounded functions where all classifiers are in $\bar{\calC}$. Even though $\calC$ consists of only binary classifiers, $\calC_B$ has real-valued predictors. Also note that most families of classifiers contain the constant 1 predictor $c(x)=1$, in which case $\calC_B$ can fit a bias term.

Let $\rho$ be a distribution over $\calZ$. Let $\calZ_1,\ldots,\calZ_K \subseteq \calZ$ be $K$ groups. For $h \in \calC_B$, let, 
$$\ell(h) = \left\langle\E_{x,y\sim\rho}\bigl[(h(x)-y)^2~\bigl|~(x,y) \in \calZ_k \bigr]\right\rangle_{k=1}^K.$$


Key to the reduction, Algorithm \ref{alg:regress} (analyzed in Section \ref{ap:learningProofs}) solves regression using a classification algorithm. With this, Corollary \ref{cor:regression} is a straightforward application of Theorem \ref{thm:group} for nondecreasing $f$, with a choice set of $\calC_B$ and comparison set of $\bar{\calC}_B$.

\begin{corollary}\label{cor:regression}
Let $M$ be an efficient agnostic learner for $\bar{\calC}$. Let $B, L>0$ and $f:\reals_+^K \rightarrow \reals$ be nondecreasing and $L$-Lipschitz. 
For any $\eps, \delta\in (0,1)$, with probability $\geq 1-\delta$, Algorithms \ref{alg:groupopt} and \ref{alg:regress} with knowledge of group probabilities $p_k=\rho(\calZ_k)$ can be used to output $h \in \calC_B$ such that,
$$f(\ell(h))\leq \min_{l\in \ell(\bar{\calC}_B)} f(l)+\eps.$$
Let $\pmin = \min_k p_k$ be the smallest group probability. For any constant $K$, the algorithm uses $\poly(B, L, 1/\eps, 1/\delta, 1/\pmin)$ total runtime, examples, calls to group membership oracles $\calZ_k$, and calls to agnostic learner $M$. 
\end{corollary}
The simple proof follows from Theorem \ref{thm:group} and Lemma \ref{lem:classification-to-regression} which shows that Algorithm \ref{alg:regress} reduces regression (with example weights) to agnostic learning.

\begin{algorithm}
\caption{\textsc{ClassyRegression}: Weighted regression using binary classification}
\label{alg:regress}
\KwIn{Agnostic learner $M$, weighted examples $Z \in \bigl([0,1] \times \calX \times [-1,1]\bigr)^*$, $\eps>0$, $B > 0$}
\KwOut{$h: \calX \rightarrow \reals$}

\smallskip
Let
$T=\frac{12B^2}{\eps}\ln \frac{12B^2}{\eps}$,
batch size $n=\frac{|Z|}{2T}$.

\smallskip
Pick $c_1 \in \calC, \alpha_1=B$ ~~~~// arbitrary initial classifier

\For {$t=1$ to $T$} {

      $h_t = \frac{1}{t}(\alpha_1 {c_1}+\cdots+\alpha_t {c_t}) ~~~~\text{// }h_t: \calX \rightarrow [-B, B]$
      
      Take $2n$ fresh examples from $Z$, call them $(w_1,x_1,y_1), \ldots, (w_{2n}, x_{2n},y_{2n})$.
      
      For $i=1, \ldots, n$, let $y_i'=1$ with probability
      $\frac{1}{2}+\frac{w_i}{2(B+1)}\bigl(y_i - h_t(x_i)\bigr)$, and $y'_i=0$ otherwise.
            
      \smallskip
      $\mathrm{Candidates}_t = \bigl\{(B, M(x_1,y_1', \ldots, x_n, y_n')), (-B, M(x_1,1-y_1', \ldots, x_n, 1-y_n'))\bigr\}$
      
      \smallskip
 Choose $(\alpha_{t+1}, c_{t+1}) \in \mathrm{Candidates}_t$ with greater $\sum_{i={n+1}}^{2n} \alpha_{t+1} {c_{t+1}}(x_i)w_i\bigl(y_i - h_t(x_i)\bigr)$
}   


\Return $h_T$
\end{algorithm}

\subsubsection{Specific objectives}

In this section, we illustrate how existing objectives are group losses. Eq.~(\ref{eq:constraints}) shows how a continuous constraint $g(l) \leq 0$ (which itself can capture multiple constraints) can be folded into a loss function. If $g$ is satisfiable, then we find $c$ with $g(c) \leq \eps$ in time $\poly(1/\eps)$. In many learning applications, one cannot guarantee perfect satisfiability. For instance a constraint of equal error rates among men and women can only be approximately checked.

Equal false positive constraints could be modeled as $\sum_{k,k'} |\FP_k(c)-\FP_k'(c)|$, as a maximum over deviations from the mean, or through other approaches. Equal error rates amongst groups could be modeled as $\sum_{k,k'} |\ERR_k(c)-\ERR_{k'}(c)|$, and eq.~(\ref{eq:group-err}) shows how error rates can be written in terms of false positive and false negative rates. 

Since it is known that one cannot simultaneously satisfy many natural notions of fairness \citep{Chouldechova17,KleinbergMR17}, it may be desirable to specify an objective that trades off these various notions.  
Equations (\ref{eq:group-err})-(\ref{eq:f1}) show how various natural quantities such as error, precision, recall, and the F1 score can be written in terms of the false positive and false negative group rates. 

One issue that arises is that the $\FONE$ score is discontinuous when the false negative rate is 1 because the recall and precision are 0.   However, one can replace $\FONE$ by a version which is $L$-Lipschitz continuous and agrees with $\FONE$ outside of a small region in which the false negative rate is close to 1. 

Consider the objective in the introduction, eq.~(\ref{eq:example-obj}) but more generally for $K$ groups. It involves the following quantity,
$$\Pr[x \in \calX_k~|~c(x)=1]=\frac{p_k^0\FPR_k(c)+p_k^1(1-\FNR_k(c))}{\sum_{j}p_j^0\FPR_j(c)+p_j^1(1-\FNR_j(c))},$$
for each group $k$. Note that this is also discontinuous when the classifier is always 0, however replacing the denominator $d$ by $\max(d, \delta)$ for a small $\delta$ makes the function continuous and does not change the objective except when the classifier is nearly trivial in the sense that $c(x)=0$ almost all the time. A similar approximation can be done for other losses that are not Lipschitz-continuous.

\section{Convex Losses}\label{sec:convex}

In this section, we use $n$ to denote the number of components to reflect the fact that there may be many components, so $\ell:\calC\rightarrow [0,1]^n$. \cite{NarasimhanRS015} show how to optimize convex performance functions of confusion matrices. Our algorithm and analysis in this section extends this to the case of multiple, possibly-overlapping, groups.

As mentioned, the reduction from a linear optimizer to a convex optimizer is a standard Frank-Wolfe style of reduction. In this section, rather than reinvent the wheel, we appeal to an existing result that meets our needs. To do so, we make some further simplifying assumptions.
In this section we assume $\bar{\calC}= \calC$, i.e., the optimizer outputs an element of the comparison set. To use the results of \cite{KakadeKL09}, we also henceforth assume that $\ell(\calC)=\{\ell(c)|c \in \calC\}$ is a compact set. This assumption conveniently lets one assume that the minimum of $f$ is attained for some choice $\bar{c} \in \calC$.

Consider convex $f:[0,1]^n\rightarrow [0,1]$ with a gradient or subgradient oracle, denoted $\nabla f:[0,1]^n\rightarrow \reals^n$, that computes a subgradient at each point $l$. The goal is to minimize $f(\ell(c))$ over $c\in \calC$, possibly subject to constraints. For differentiable $f$, the gradient is the  unique tangent plane at every point. More generally, every convex function (differentiable or non-differentiable) has at least one sub-gradient at every point. Hence, we assume that we have a subgradient oracle $\nabla f:[0,1]^n\rightarrow [0,1]^n$ which is just the gradient of $f$ if it is differentiable, and more generally satisfies,
\begin{equation}\label{eq:subgradient}
\nabla f(l) \cdot (l'-l) \leq f(l')-f(l) \text{ for all }l, l' \in [0,1]^n.
\end{equation}
Corollary \ref{cor:constrained} adds convex constraints to the problem.

\begin{observation}\label{obs:convex}
Let $f:[0,1]^n\rightarrow \reals$ be a (nondecreasing) convex function with subgradient oracle $\nabla f:[0,1]^n \rightarrow [-1,1]^n$ ($\nabla f:[0,1]^n \rightarrow [0,1]^n$).
For any $\eps>0$, Algorithm 3.1 of \cite{KakadeKL09} with (nonnegative) linear optimizer $M_\tau$  finds $c \in \calC$ such that,
$$f(\ell(c)) \leq \min_{\bar{c} \in \calC} f(\ell(\bar{c})) + \eps.$$
The runtime and number of calls to $M_\tau$, $\ell$, and $\nabla f$ are $\poly(1/\eps, n)$ with $\tau=\eps/6$.
\end{observation}
Algorithm 3.1 of \cite{KakadeKL09} is an online algorithm that can be used to minimize a sequence of functions $f_1, \ldots, f_T$, with a regret bound $\eps$ meaning that the algorithm chooses a sequence of points $c_1, c_2, \ldots, c_T$ such that,
$$\frac{1}{T} \sum_{t=1}^T f_t(c_t)  \leq \min_{\bar{c} \in \calC} \frac{1}{T} \sum_{t=1}^T f_t(\bar{c})+\eps.$$  
Taking $f_1=f_t=\cdots=f$, convexity of $f$ gives a bound on $f(c)\leq \min_{\bar{c}\in\calC} f(\bar{c}) + \eps$ for $c=\mathrm{UniformDist}(\{c_1, \ldots, c_T\})$. However, each $f_t$ in Algorithm 3.1 must be linear. To linearize, we take each $f_t$ to be the tangent plane of $f$ at $c_{t-1}$ which is a lower-bound on $f$. The proof of Observation \ref{obs:convex} is found in Section \ref{sec:convex-proof} of the appendix. 

To fold constraints into the objective $f$, one needs to be able to compute subgradients. This is made explicit in the following corollary. 
\begin{corollary}\label{cor:constrained}
Let $f, g:[0,1]^n\rightarrow \reals$ be (nondecreasing) convex functions with subgradient oracles $\nabla f, \nabla g:[0,1]^n \rightarrow [-1,1]^n$ ($[0,1]^n$) such that $\exists \bar{c}\in \calC$ with $g(\ell(\bar{c})) \leq 0$. For any $\eps>0$, Algorithm 3.1 of \cite{KakadeKL09} with (nonnegative) linear optimizer $M_\tau$ can be used to find $c \in \calC$ such that $g(c)\leq \eps$ and 
$$f(\ell(c)) \leq \min_{\bar{c} \in \calC: g(\ell(\bar{c}))\leq 0} f(\ell(\bar{c})) + \eps.$$
The runtime and number of calls to $M_\tau$, $\ell$, and $\nabla f, \nabla g$ are $\poly(1/\eps, n)$ and $\tau=\eps^2/(12\eps + 6n)$.
\end{corollary}
\begin{proof}
Let $\lambda = \eps/(2\eps+n)$. We apply Observation \ref{obs:convex} to minimize $h(l)= \lambda f(l)+(1-\lambda) \max(0, g(l))$ to within $\eps'=\lambda \eps = \eps^2/(2\eps + n)$.
From eq.~(\ref{eq:constraints}), since the range of $\lambda f$ is at most $R \leq \lambda n$  (since the diameter of $[0,1]^n$ is $\sqrt{n}$ and $\|\nabla f(l)\| \leq \sqrt{n}$), we have that
an $\eps'$ optimal minimizer of $h$ have $$g(l) \leq \frac{\eps' + \lambda n}{1-\lambda} = \eps.$$
Note that $\ind[g(l) \geq 0]\nabla g(l)$ is a subgradient of $\max(0, g(l))$ if $\nabla g(l)$ is a subgradient of $g(l)$. Thus we can efficiently compute subgradients $\nabla h$ as well. Clearly $\nabla h(l) \in [-1,1]^n$ ($[0,1]^n$) as well.
\end{proof}
Using the same approach, constraints can be folded into the objectives in the other theorems and corollaries in this paper.

\subsection{Applications}

We now highlight two applications of the observation and corollary in the previous section. One convex loss would be,
$$f(l) = \lambda \sum_i l_i + (1-\lambda) \max_i l_i,$$
which has a parameter $\lambda \in [0,1]$. For $\lambda =0$, this is the maximum in the spirit of max-min fairness \cite{rawls1971theory}. For $\lambda =1$, this is average performance.  Since $\max$ is nondifferentiable, there are many subgradients $\nabla f: [0,1]^n\rightarrow \reals^n$, and any will do. One is 
$$\nabla f_i(l)=\lambda + (1-\lambda) \ind\left[l_i =\max_j l_j\right].$$
Also note that this objective is nondecreasing. Hence, a FLOP algorithm could be used to minimize this objective which, for $\lambda \in [0,1]$, trades off the average customer's distance to nearest facility with the maximum customer's (expected) distance to nearest facility.

Finally, consider the Gini index of inequality \citep{gini1912variabilita},
$$G(l)=\frac{\sum_{i,j} |l_i-l_j|}{2n\sum l_i} \in [0,1].$$ While the Gini index is not convex itself, it is quasi-convex meaning that it's level sets are convex. In particular, for any given $\theta\in [0,1]$, $G(l) \leq \theta$ is equivalent to the convex constraint,
$$\sum_{i,j} |l_i-l_j|-2n\theta \sum l_i\leq 0.$$
This observation can be used to minimize any given quantity such as $\lambda \sum_i l_i + (1-\lambda) G(l)$. In particular, one can solve the constrained optimization problem,
$$\min_{l\in F:\sum |l_i-l_j|-2n\theta \sum l_i \leq 0} \sum_i l_i,$$ 
for a fine grid of $\theta \in [0,1]$. Each constrained problem for $\theta\in [0,1]$ can each be solved using Corollary \ref{cor:constrained}, and the resulting solution can be checked to see if $G(l)\leq \theta$.

\section{Conclusions}

This paper enables application designers and data scientists to optimize complex continuous objectives, with possibly many local non-global minima, that account for differences across various groups. It is not our claim or belief that important issues such as fairness can be captured solely by such group statistics. However, with this ability, domain experts are hopefully better equipped to design objectives that, when optimized, lead to desirable outcomes.

\medskip \noindent \textbf{Acknowledgments}. We would like to thank Miro Dudík, Salil Vadhan, and the anonymous COLT reviewers for helpful comments.

\bibliography{colt2018}

\section{Appendix}
Below we present the proofs of our theorems.

\subsection{Proof of Theorem \ref{thm:group}}\label{ap:proof:thm:group}

For some readers, intuition may be gained by first reading the proof for the  case $N=0$, i.e., where $f$ may be decreasing in places so $M_\tau$ can take negative weights, and $\hat{l}_t = \frac{1}{t}(\ell_\tau(c_1)+\cdots+\ell_\tau(c_t))$. We first show that for any grid point $r$, the inner loop of the algorithm does approximately minimize $\|\hat{l}_T-r\|$. If $\bar{\calC}= \calC$, this lemma essentially says that $\hat{l}_T$ is very close to the projection of $r$ onto $\ell(\calC)$. 

\begin{lemma}\label{lem:finally}
Fix any grid point $r$ visited in the Algorithm \ref{alg:groupopt} \textsc{GroupOpt}. Let $s_t = \ell(c_t)$ and
$l_t = \max(\frac{1}{t}(s_1+\cdots+s_t), Nr)$ for $t=1,2,\ldots, T$. Then,
$$(l_t-r)^2\leq \min_{\bar{c}\in \bar\calC}(\ell(\bar{c})-r)^2 + 4\tau\sqrt{K}+\frac{K(1+\ln t)}{t}.$$
\end{lemma}
\begin{proof}
Since $\max(Nr, \cdot)$ is a contraction map,
\begin{align}
\|l_t-\hat{l}_t\| &= \left\|\max\left(Nr,\frac{1}{t}(s_1+\cdots+s_t)\right)-\max\left(Nr,\frac{1}{t}(\ell_\tau(c_1)+\cdots + \ell_\tau(c_t)\right)\right\|\nonumber\\
&\leq \left\|\frac{1}{t}(s_1+\cdots+s_t)-\frac{1}{t}(\ell(c_1)+\cdots+\ell(c_t))\right\|\nonumber\\
&\leq \frac{1}{t}\bigl(\norm{s_1-\ell(c_1)}+\cdots+\norm{s_t-\ell(c_t)}\bigr) \leq \tau,  \label{eq:e0}
\end{align}
using $\|s_i-\ell_\tau(c_i)\|\leq \tau$ by definition of $\ell_\tau$.
We next claim that for $t>1$, 
\begin{equation}\label{eq:e1}
(l_t-r)^2 \leq \left(\left(1-\frac{1}{t}\right)(l_{t-1}-r) +\frac{1}{t}(s_t-r)\right)^2.
\end{equation}
For $N=0$, the above is an equality because $l_t=\left(1-\frac{1}{t}\right)l_{t-1} +\frac{1}{t}s_t$. For $N=1$, one can verify the above coordinate-wise using the fact that for any $x,y\in \reals$, $$\bigl(\max(x+y, 0)\bigr)^2 \leq \bigl(\max(x+y, y)\bigr)^2 = \bigl(\max(x,0)+y\bigr)^2.$$ To get (\ref{eq:e1}), sum over coordinates $k\in [K]$ and substitute $x=(s_{1k}+\cdots+s_{(t-1)k}-(t-1)r_k)/t$ and $y=(1/t)(s_{tk}-r_k)$ for each $k$.

Let $u=\ell(\bar{c})$ for any $\bar{c} \in \arg\min_{\bar\calC}(\ell(\bar{c})-r)^2$, i.e., $u$ is the closest point to $r$ in $\ell(\bar\calC)$. By definition of $M_\tau$, 
\begin{align*}
(l_{t-1}-r)(s_{t}-u) &= (\hat{l}_{t-1}-r)(s_t-u) + (l_{t-1}-\hat{l}_{t-1})(s_{t}-u)\\
&\leq \tau\|\hat{l}_{t-1}-r\| + \norm{l_{t-1}-\hat{l}_{t-1}}\cdot \norm{s_{t}-u}
\end{align*}
But since $\|l_{t-1}-\hat{l}_{t-1}\|\leq \tau$ and $\hat{l}_t,l_t, r, u, s_t \in [0,1]^K$, and the diameter of $[0,1]^K$ is $\sqrt{K}$, the above is at most $2\tau\sqrt{K}$, or written differently:
\begin{equation}
(l_{t-1}-r) s_t \leq (l_{t-1}-r)u +  2\tau\sqrt{K}.\label{eq:e2}
\end{equation}
Next, (\ref{eq:e1}) can be rewritten as,
\begin{align*}
(l_t-r)^2 &\leq \left(l_{t-1}-r + \frac{1}{t}(s_t-l_{t-1})\right)^2\\
&=(l_{t-1}-r)^2 + \frac{2}{t}(l_{t-1}-r)(s_t-l_{t-1}) + \frac{1}{t^2}(s_t-l_{t-1})^2\\
&\leq  (l_{t-1}-r)^2 + \frac{2}{t}(l_{t-1}-r)(s_t-l_{t-1}) + \frac{K}{t^2} &\text{// since }s_t, l_{t-1} \in [0,1]^K\\
&\leq (l_{t-1}-r)^2 + \frac{2}{t}(l_{t-1}-r)({\color{blue} u}-l_{t-1}) +{\color{blue} \frac{2}{t} 2\tau\sqrt{K}} + \frac{K}{t^2}    &\text{// by (\ref{eq:e2})}\\
&\leq \left(l_{t-1}-r + \frac{1}{t}(u-l_{t-1})\right)^2 +\frac{2}{t}2\tau\sqrt{K} + \frac{K}{t^2}\\
&= \left(\left(1-\frac{1}{t}\right)(l_{t-1}-r) + \frac{1}{t}(u-r)\right)^2 +\frac{4\tau\sqrt{K}}{t} + \frac{K}{t^2}\\
&\leq \left(1-\frac{1}{t}\right)(l_{t-1}-r)^2 + \frac{1}{t}(u-r)^2 +\frac{4\tau\sqrt{K}}{t} + \frac{K}{t^2} &\text{// by convexity of $(\cdot)^2$}
\end{align*}
Multiplying both sides by $t$, 
$$t(l_t-r)^2 \leq (t-1)(l_{t-1}-r)^2 + (u-r)^2 + 4\tau\sqrt{K} + \frac{K}{t}$$ 
To finish, the above implies, by induction on $\alpha_t = t(l_t-r)^2$ over $t$, 
$$\alpha_t = t(l_t-r)^2 \leq t(u-r)^2 + 4\tau t\sqrt{K} + K\left(1+\frac{1}{2}+\cdots+\frac{1}{t}\right).$$
(Note the base case $\alpha_1=(l_1-r)^2 \leq K$.)
Dividing both sides by $t$ and using the fact that $H_t \leq 1+ \ln(t)$ gives the lemma.
\end{proof}
As mentioned, the above lemma is essentially a re-analysis of the Frank-Wolfe algorithm \citep{revisitingFrankWolfe}
 that handles a comparison set $\bar{\calC}$.

\begin{proof}[Theorem \ref{thm:group}]
Let $\bar{l} \in \arg\min_{\bar{l} \in \ell(\bar\calC)} f(\bar{l})$ be any optimum. Let $\bar{r}$ be a grid point that is within $q\sqrt{K}/2=\beta/2$ of $\bar{l}$, i.e., $\|\bar{r}-\bar{l}\|\leq \beta/2$, which must exist since the grid spacing is $q$. 

We first argue that the outer for loop will not pass $\bar{r}$, i.e., it will return on or before the iteration $r=\bar{r}$. To see this, suppose it reaches $r=\bar{r}$. 
For $l_t$ as defined in the statement of Lemma \ref{lem:finally}, we have that at the end of inner loop, $$(l_T -\bar{r})^2\leq \beta^2/4+4\tau\sqrt{K} + K(1+\ln(T))/T \leq \beta^2 + 4\beta^2 + 2 \beta^2= 7\beta^2,$$
since $4\tau\sqrt{K}= 4\beta^2$ and by our choice of $T$, $K(1+\ln(T))/T \leq 2\beta^2$. Again using $\|l_{T}-\hat{l}_{T}\|\leq \tau$ from (\ref{eq:e0}), this implies,
$$\|\hat{l}_{T}-\bar{r}\| \leq \|l_{T}-\bar{r}\| + \|l_{T}-\hat{l}_{T}\| \leq \sqrt{7}\beta + \tau \leq 3\beta,$$
since $\tau \leq \beta/4$ for $\eps\leq 1$.

Hence we conclude that the algorithm will output some distribution on or before it reaches $r=\bar{r}$. Fix the $r$ for which it returns and let $c$ be the uniform distribution over $c_1,\ldots, c_T$ returned. Let $s_t, l_t$ be as defined in the statement of Lemma \ref{lem:finally}, so $\ell(c)=(s_1+\cdots+s_T)/T$ and $l_T=\max(Nr, \ell(c))$. Since the grid points are sorted, $f(r) \leq f(\bar{r})$. 

To analyze $f(\ell(c))$, the definition of Lipschitz implies, $f(l_T) - f(r) \leq \norm{r-l_T}$, which is $\leq \|\hat{l}_T-r\| + \|\hat{l}_T - l_T\| \leq 3\beta + \tau$. Now, if $N=0$ then $\ell(c)=l_T$ and $f(\ell(c))\leq f(l_T)$. If $N=1$ then $\ell(c)\leq l_T$, and if $f$ is nondecreasing then $f(\ell(c)) \leq f(l_T)$. 
By the Lipschitz condition, and again using the inequality $\norm{l_T-r}\leq 3\beta+\tau$, 
$$f(\ell(c))\leq f(l_T) \leq f(r) + \norm{l_T-r} \leq f(\bar{r}) + 3\beta+\tau \leq f(\bar{l})+\norm{\bar{l}-\bar{r}} + 3\beta+\tau.$$
Since $\norm{\bar{r}-\bar{l}}\leq \beta/2$ and $3.5\beta+\tau<\eps$, this would seem to establish $f(\ell(c)) \leq f(\bar{l})+\eps$ as needed. However, there is a subtle case not covered above in which $N=1$ because $f$ happens to be nondecreasing on the grid but $f$ may be in fact decreasing in places. It is not hard to see that there must be another grid point $r'\leq r$ which is within $3.5\beta+\tau$ of $\ell(c)$, since every point is within $\beta/2$ of some grid point. Moreover, since $f$ is nondecreasing on the grid, $f(r')\leq f(r)$ and the analysis goes through as above with an additional $\beta/2$, which is okay because $4\beta + \tau < \eps$. 

To bound the runtime, note that $|G|\leq q^{-K} = (5/\eps)^K K^{K/2}$, $T = 25K/\eps^2\ln(25K/\eps^2)$, and the number of iterations of the inner loop is at most $|G|T$. The number of calls to each oracle is $\leq |G|T$ (assuming we memoize calls to $\ell_\tau$).
\end{proof}

\subsection{Learning Proofs}\label{ap:learningProofs}
\begin{proof}[Proof of Corollary \ref{cor:classification}]
WLOG $L=1$. For $L>1$, we can apply the theorem to $f/L$ with accuracy $\eps/L$, and $f/L$ is Lipschitz. We apply Theorem \ref{thm:group} to the $K'=2K$ groups $\calZ_k^i = \calX_k \times \{i\} \subseteq \calZ$ for $k\in [K], i\in\{0,1\}$ and loss 
$$\ell_{k}^i(c)=\Pr[c(x)\neq y~|~(x, y) \in \calZ_k^i].$$
Note that $\ell_k^0(c)=\FPR_k(c)$ and $\ell_k^1(c)=\FNR_k(c)$. Hence, all that remains is to show how to implement $\ell_\tau$ and $M_\tau$ for the $Q=\poly(1/\eps)$ total calls to $\ell_\tau$ and $M_\tau$ for classification and regression, for $\tau=\eps^2/(25\sqrt{K})$. By Chernoff bounds, with $\poly(\log 1/\delta, 1/\eps, 1/\pmin)$ examples, clearly we can estimate each of the $T$ calls to $\ell_\tau$ to within $\tau$ accuracy with probability $\geq 1-\delta/Q$.  This is because an expected $1/\pmin$ examples suffice to get an example from any group $\calZ_{k,i}$.

We now describe how to simulate $M_\tau(w)$ for any $w =\langle w_k^i\rangle_{k,i} \in \reals^{2K}$. Let normalization $W = \sum_{k,i} |w^i_k|/p^i_k $, and define $s(x, y, w) = s(z,w) = (1/W) \sum_{k: z \in \calZ_k^i} w^i_k/p_k^i \in [-1,1]$. In these terms, $w \cdot \ell(c) = \E[W s(z,w) |c(x)-y|]$ and hence linear optimizer $M_\tau(w)$ must output $c$ such that,  
\begin{equation}\label{eq:3324}
\E[s(x,y,w) |c(x)-y|] \leq \min_{\bar{c} \in \bar\calC} \E[s(x,y,w) |\bar{c}(x)-y|]+\frac{\tau \|w\|}{W}.
\end{equation}


We use a standard label-flipping trick to simulate positive or negative weights. In particular, consider the distribution $\mu'$ which samples $x,y$ from $\mu$ and outputs $x,y'$ where $y'=y$ with probability $(1+s(x,y,w))/2$ and $y'=1-y$ otherwise. Then, for any $c\in \calC$, the error with respect to $\mu'$ can be written as
\begin{align*}
\E_{x,y' \sim \mu'}[|c(x)-y'|] &= \E_{x,y \sim \mu}\left[\frac{1+s(x,y,w)}{2}|c(x)-y|+\frac{1-s(x,y,w)}{2}(1-|c(x)-y|)\right] \\
&= \E_{x,y \sim \mu}\left[s(x,y,w)|c(x)-y|+\frac{1-s(x,y,w)}{2}\right] 
\end{align*}
This means that for any $c,\bar{c}$,
$$\err(c, \mu')-\err({\bar{c}}, \mu') = \E_{x,y\sim \mu}[s(x,y,w) |c(x)-y|-s(x,y,w) |\bar{c}(x)-y|].$$
So if we feed sufficiently many examples from $\mu'$ (which is easy to simulate given examples from $\mu$) into $M$, with probability $\geq 1-\delta/Q$, $M$ will output $c$ with $\err(c, \mu')-\err({\bar{c}}, \mu')\leq \tau \pmin/\sqrt{K}$ which suffices for (\ref{eq:3324}) because $\tau \pmin/\sqrt{K}\leq \tau\norm{w}/W$ since $W\leq \norm{w}_1/\pmin \leq \norm{w}\sqrt{K}/\pmin$.

Because we are only changing labels, the underlying distribution on $\calX$ doesn't change. By the union bound, all $Q$ calls to $\ell_\tau$ and $M_\tau$ can be answered within $\tau$ tolerance with probability $\geq 1-\delta$ as needed. The calls to the agnostic learning algorithm had accuracy $\eps'=\tau\pmin /\sqrt{K}$ and confidence $\delta'=\delta/Q$ which means that the total runtime is $\poly(1/\eps, 1/\pmin)$.
\end{proof}

Algorithm \ref{alg:regress} repeatedly calls a linear classifier to provably minimize squared error over the family of (bounded) linear combinations of classifiers. We also require nonnegative example weights, since minimizing squared error with negative weights requires exponential time in general.

\begin{lemma}\label{lem:classification-to-regression}
Let $\sigma$ be a distribution over $[0,1]\times \calX \times [-1,1]$, $M$ be an agnostic learner for $\bar{\calC}, \calC$.
Then,  for any $\eps, \delta \in (0,1)$ and any $B>0$, Algorithm \ref{alg:regress} outputs a predictor $h \in \calC_B$ such that, with probability $\geq 1-\delta$, 
$$\E_{w,x,y\sim \sigma}\bigl[w(y-h(x))^2\bigr] \leq \min_{\bar{h} \in \bar{\calC}_B} ~\E_{w,x,y\sim \sigma}\left[w\left(y-\bar{h}(x)\right)^2\right]+\eps.$$
for $\poly(B, 1/\eps, 1/\delta)$ examples and calls to $M$.
\end{lemma}
\begin{proof}
Take $\bar{h}\in \bar{C}_B$ that minimizes $\E[(\bar{h}(x)-y)^2]$. 
For the moment, fix any $t$. Let $\mathrm{Candidates}_t=\{(B,c_a), (-B,c_b)\}$. For any $w,x,y\in [0,1] \times \calX \times\{0,1\}$, imagine picking $y'=1$ with probability $\frac{1}{2}+\frac{w}{2(B+1)}\bigl(y - h_t(x)\bigr)$ and 0 otherwise, as in step $t$ of the algorithm. Note that for any $c: \calX \rightarrow \{0,1\}$ and any $w,x,y$,
\begin{align*}
\E_{y'}\bigl[|c(x)-y'|\bigr]&=\left(\frac{1}{2}+\frac{w}{2(B+1)}\bigl(y - h_t(x)\bigr)\right)(1-c(x)) +
\left(\frac{1}{2}-\frac{w}{2(B+1)}\bigl(y - h_t(x)\bigr)\right)c(x)\\
&= \left(\frac{1}{2}+\frac{w}{2(B+1)}\bigl(y - h_t(x)\bigr)\right) - \frac{w}{(B+1)}\bigl(y - h_t(x)\bigr)c(x)
\end{align*}
Note the left term above is independent of $h$. Let ${\bar{c}}\in\bar{\calC}$ be any comparison class classifier. Using this identity for $c \in \{c_a,\bar{c}\}$, and letting $\mu$ be the distribution of $x,y'$ generated for $w,x,y\sim \sigma$, 
\begin{align*}
\err(c_a,\mu)-\err(\bar{c},\mu) &= 
\E_{x,y'\sim \mu}\bigl[|c_a(x)-y'|-|\bar{c}(x)-y'|\bigr]\\
&=\E_{w,x,y\sim \sigma}\left[\frac{w}{(B+1)}\bigl(y - h_t(x)\bigr)\bigl(\bar{c}(x)-c_a(x)\bigr)\right]\\
&\leq\E_{w,x,y\sim \sigma}\left[\frac{w}{B}\bigl(y - h_t(x)\bigr)\bigl(\bar{c}(x)-c_a(x)\bigr)\right]
\end{align*}
By the definition of agnostic learner, for $n = \poly(T, 1/\delta)$, with probability $\geq 1-\delta/(4T)$, this implies that $M$ called on $n$ examples from $\mu$ will have $\err(c_a,\mu)-\err(\bar{c},\mu)\leq 1/2T$, which can be written as:
\begin{equation}
\E_{w,x,y\sim \sigma}[w(y-h_t(x))c_{a}(x)]\geq \max_{\bar{c} \in \bar{\calC}} \E_{w,x,y\sim \sigma}[w(y-h_t(x))\bar{c}(x)] - \frac{B}{2T}\label{eq:yoyo}
\end{equation}
An entirely similar argument for the distribution over $(x,1-y')$ implies that,
\begin{equation}\label{eq:maya}
\E_{w,x,y\sim \sigma}[-w(y-h_t(x))c_{b}(x)]\geq \max_{\bar{c} \in \bar{\calC}} \E_{w,x,y\sim \sigma}[-w(y-h_t(x))\bar{c}(x)] - \frac{B}{2T}
\end{equation}
By Chernoff bounds for $n=\poly(B,1/T,\log 1/\delta)$, with probability $\geq 1-\delta/(2T)$, the choice for $(\alpha_{t+1}, c_{t+1}) \in \{(B,c_a), (-B,c_b)\}$ which has greater empirical correlation $\sum \alpha_{t+1} c_{t+1}(x_i)w_i\bigl(y_i - h_t(x_i)\bigr)$ on $n$ fresh iid examples from $\sigma$ will be within $B^2/T$ of the greater correlated of $\{Bc_a, -Bc_b\}$, i.e.,
$$\E\left[w(y-h_t(x))\alpha_{t+1}c_{t+1}(x)\right] \geq B\max\bigl(\E[w(y-h_t(x))c_{a}(x)], \E[-w(y-h_t(x))c_{b}(x)]\bigr) - \frac{B^2}{T}$$
Now, the above together with (\ref{eq:yoyo}) and (\ref{eq:maya}) and the fact that the maximum over a linear function always occurs at vertex, we get that 
\begin{align*}
\E\left[w(y-h_t(x))\alpha_{t+1}c_{t+1}(x)\right] &\geq \max_{(\bar\alpha, \bar{c}) \in [-B, B] \times \bar{\calC}} ~\E[w\bar\alpha \bar{c}(x)(y-h_t(x))]-\frac{B^2}{T}\\
&=\max_{\bar{h}\in \bar{\calC}_B} \E_{x, y \sim \sigma}[w\bar{h}(x)(y-h_t(x))]-\frac{B^2}{T}.
\end{align*}
We now stop fixing $t$ and look across $t$.  For notational ease, let $s_t(x)=\alpha_t {c_t}(x)$ so the above implies:
\begin{equation}
\E\left[w(y-h_t(x))s_{{t+1}}(x)\right] \geq
\max_{\bar{h}\in \bar{\calC}_B}\E_{x, y \sim \sigma}[w(y-h_t(x))\bar{h}(x)]-\frac{B^2}{T}\label{eq:katie}
\end{equation}
In these terms, we also have $h_t=(1-1/t)h_{t-1} + (1/t)s_t$. In particular, similar to the proof of Lemma \ref{lem:finally},
\begin{align*}
(h_t(x)-y)^2&=\left(h_{t-1}(x)-y+\frac{1}{t}(s_t(x)-h_{t-1}(x))\right)^2\\
&=(h_{t-1}(x)-y)^2+\frac{2}{t}(h_{t-1}(x)-y)(s_t(x)-h_{t-1}(x))+\frac{1}{t^2}(s_t(x)-h_{t-1}(x))^2\\
&\leq(h_{t-1}(x)-y)^2+\frac{2}{t}(h_{t-1}(x)-y)(s_t(x)-h_{t-1}(x))+\frac{4B^2}{t^2}\\
\end{align*}
We now use (\ref{eq:katie}) (negated and for $s_t$) and the fact that $w \in [0,1]$:
\begin{align*}
\E[w(h_t(x)-y)^2]
&\leq\E\left[w(h_{t-1}(x)-y)^2 + \frac{2}{t}w(h_{t-1}(x)-y)({\color{blue} s_t(x)}-h_{t-1}(x)) + \frac{4B^2}{t^2}\right]\\
&\leq\E\left[w(h_{t-1}(x)-y)^2 + \frac{2}{t}w(h_{t-1}(x)-y)({\color{blue} \bar{h}(x)}-h_{t-1}(x)) +{\color{blue} \frac{2}{t} \frac{B^2}{T}} + \frac{4B^2}{t^2}\right]\\
&\leq\E\left[w(h_{t-1}(x)-y)^2 + \frac{2}{t}w(h_{t-1}(x)-y)(\bar{h}(x)-h_{t-1}(x)) +\frac{6B^2}{t^2}\right].
\end{align*}
We now complete the square in the first two terms above, giving,
\begin{align*}
\E[w(h_t(x)-y)^2]&\leq \E\left[w\left((h_{t-1}(x)-y) + \frac{1}{t}(\bar{h}(x)-h_{t-1}(x))\right)^2 + \frac{6B^2}{t^2}\right]\\
&\leq \E\left[w\left(\left(1-\frac{1}{t}\right)(h_{t-1}(x)-y) + \frac{1}{t}(\bar{h}(x)-y)\right)^2 + \frac{6B^2}{t^2}\right].
\end{align*}
Using the convexity of $(\cdot)^2$, this is at most
\begin{align*}
\E[(h_t(x)-y)^2]&\leq \E\left[\left(1-\frac{1}{t}\right)(h_{t-1}(x)-y)^2 + \frac{1}{t}(\bar{h}(x)-y)^2 +\frac{6B^2}{t^2}\right]\\
t\E[(h_t(x)-y)^2]&\leq \E\left[\left(t-1\right)(h_{t-1}(x)-y)^2 + (\bar{h}(x)-y)^2  + \frac{6B^2}{t}\right]
\end{align*}
By induction on $t$ with $\beta_t = t\E[(h_t(x)-y)^2]$,
$$\beta_T = T \E[(h_T(x)-y)^2] \leq T \E[(\bar{h}(x)-y)^2] + 6B^2\left(1+\frac{1}{2} + \ldots + \frac{1}{T}\right)$$
Finally, dividing by $T$, we get,
$$\E[(h_T(x)-y)^2] \leq \E[(\bar{h}(x)-y)^2] +  \frac{6B^2}{T}H_T.$$
Using the fact that harmonic sums $H_T \leq 1+ \ln(T)$ gives the inequality in the proof for $T=\frac{12B^2}{\eps}\ln \frac{12B^2}{\eps}$ and $\eps\leq 1$.

For runtime, the two requirements on the batch size were $n=\poly(T, 1/\delta)$ for the agnostic learner and $n=\poly(B,1/T,\log 1/\delta)$ for Chernoff bounds to determine which of the two predictors was better in each round. Hence, batch size $n=\poly(B,1/\eps, 1/\delta)$ suffices and the total number of examples and calls to $M$ are $\leq 2Tn=\poly(B, 1/\eps, 1/\delta)$.
\end{proof}

\begin{proof}[Proof of Corollary \ref{cor:regression}]
Note that the loss $\ell(h)\in [0,(B+1)^2]^K$ as defined since the total weight on classifiers is at most $B$ and the labels are in $[0,1]$. Theorem \ref{thm:group} can be trivially modified to handle $\ell:\calC \rightarrow [0,(B+1)^2]^K$ and $f: [0,(B+1)^2]^K\rightarrow \reals$ simply scaling down $\ell$. 

The proof is very similar to that of Corollary \ref{cor:classification} above. Again let $Q$ be the total number of calls to $\ell_\tau$ and $M_\tau$. First, WLOG $L=1$, or else we can scale down $f$. Again, it is easy to estimate $\ell_\tau(h)$ using $\poly(B,1/\eps,\log(1/\delta), 1/\pmin)$ examples by Chernoff bounds, with probability $\geq 1-\delta/Q$. 

To compute $M_\tau(w)$ for $w \in \reals_+^K$, again define $W = \sum_k w_k/p_k$ and $s(z,w)=(1/W)\sum_{k: z \in \calZ_k} w_k/p_k$. Now $s(z,w)\in [0,1]$ and again $w \cdot \ell(h) = \E[W s(z,w) (h(x)-y)^2]$. Lemma \ref{lem:classification-to-regression} shows that Algorithm \ref{alg:regress} can be used with $M$ to find $h$ with weighed error within $\eps'=\tau \pmin/\sqrt{K}$ for weight distribution $s(z,w)$, with probability $\geq 1-\delta/Q$. This means that, compared to any $\bar{h}\in \bar{C}_B$, $h$ also satisfies,
$$w \cdot (\ell(h)-\ell(\bar{h})) =
W \E[s(x,y,w) (h(x)-y)^2 - s(x,y,w) (\bar{h}(x)-y)^2] \leq W\frac{\tau\pmin}{\sqrt{K}}\leq \tau \|w\|,$$
which is what is needed to qualify $M_\tau$ as a linear optimizer.
\end{proof}
Although not stated in Corollary \ref{cor:regression}, it is easy to see that the marginal distribution on $\calX$ for the examples fed into the agnostic learner is the same as that of $\mu$.

\section{Proof of Observation \ref{obs:convex}}\label{sec:convex-proof}
\cite{KakadeKL09} handles a more general problem in three ways. As discussed earlier, it handles an online sequence of functions. Second, it covers arbitrary restrictions on weights, not just non-negative. Third, it covers multiplicative $\alpha$ approximation algorithms. To address this last difference, we will first translate our $M_\tau$ with additive $\tau$ error into a $\alpha$-multiplicative approximation algorithm by adding a constant to our loss, so that our additive $\tau$ error translates into a multiplicative $(1+\tau)$ approximation algorithm.
\begin{proof}
We define their set $\calW=\{\frac{2}{3}\}\times [\frac{-1}{3n},\frac{1}{3n}]^n$ ($\calW=\{\frac{2}{3}\}\times [0,\frac{1}{3n}]^n$). Note that $\norm{w}\leq 1$ for $w \in \calW$.
Our feasible choices are $\calC$ (which they call decisions $\calS$). Define $\Phi:\calC \rightarrow \{1\}\times [0,1]^n$ by $\Phi(c) = (1, \ell(c))$.
Clearly $\norm{\Phi(c)}\leq \sqrt{1+n}$ and  $\Phi(c) \cdot w \in [1/3,1]$ for each $c\in \calC,w\in \calW$. As they further require (for convenience), we have assumed that $\ell(\calC)$, and hence $\Phi(\calC)$, is compact. 

Define $\alpha$-approximation algorithm $A: \calW \rightarrow \calS$ by $A(w) = A(2/3, w') = M_\tau(w)$.
For convenience in this proof, for any $w \in \calW$, we write $w' \in \reals^n$ for the last $n$ coordinates of $w$. 
 Using the fact that $\Phi(c) \cdot w = 2/3 + \ell(c) \cdot w' \geq 1/3$ for any $w \in \calW$,
the definition of (nonnegative) linear optimizer guarantees that,
$$\Phi(A(w)) \cdot w \leq 
\min_{c \in \calC} \left(\frac{2}{3} + \ell(c) \cdot w'\right) + \tau\norm{w'}\leq (1+3\tau\norm{w'})\min_{c \in \calC} \left(\frac{2}{3} + \ell(c) \cdot w'\right).$$
Since $\norm{w'} \leq 1/3$, we have that $A$ is a $\alpha$ approximation algorithm for $\alpha = 1+\tau$. 

We apply Algorithm 3.1 of \cite{KakadeKL09} to the sequence $w_1=(2/3, 0,0, \ldots, 0)$ and $w_{t+1}=(2/3, \nabla f(\ell(c_t))/3)$ where $c_t$ is the choice output by the $t$th iteration of their algorithm for $T$ iterations.  We output $\mathrm{UniformDist}\{c_1,\ldots, c_T\}$.
Since by assumption $\nabla f(\ell(c_t)) \in [-1,1]^n$ ($[0,1]^n$), $w_t \in \calW_t$. Let $\bar{c}$ be a minimizer of $f(\ell(c))$, which exists by compactness. Their Theorem 3.2 guarantees,
$$\frac{1}{T} \sum_{t=1}^T \Phi(c_t) \cdot w_t \leq (\alpha + 2)\sqrt{\frac{1+n}{T}}+\alpha \frac{1}{T} \sum_{t=1}^T \Phi(\bar{c}) \cdot w_t,$$ 
for any $\bar{c} \in \calC$, using at most $4(\alpha+2)^2T$ calls to $\Phi$ and $A$ (hence $M_\tau$).
Since $\frac{1}{T}\sum \Phi(\bar{c}) \cdot w_t\leq 1$ and $(1+n)/T\leq 1$, the above implies,
$$\frac{1}{T} \sum_{t=1}^T (\Phi(c_t)-\Phi(\bar{c})) \cdot w_t \leq (\alpha + 2)\sqrt{\frac{1+n}{T}}+\alpha -1 = 
(3+\tau) \sqrt{\frac{1+n}{T}} + \tau.$$ 
On the other hand, by convexity of $f$ and the definition of $\Phi$ and $w_t$, 
\begin{align*}
\frac{1}{T} \sum_{t=1}^T (\Phi(c_t)-\Phi(\bar{c})) \cdot w_t &= \frac{1}{3T} \sum_{t=1}^T (\ell(c_t)-\ell(\bar{c})) \cdot \nabla f(\ell(c_t)) \\
&\geq \frac{1}{3T} \sum_t f(\ell(c_t))-f(\ell(\bar{c}))  \\
&\geq \frac{1}{3}\left(f\left(\frac{\ell(c_1)+\cdots + \ell(c_T)}{T}\right)-f(\ell(\bar{c}))\right)
\end{align*}
The above is a bound on the output on $f(\ell(c))$ for $c=\mathrm{UniformDist}\{c_1,\ldots, c_T\}$. In particular, it shows $f(\ell(c)) -f(\ell(\bar{c}))\leq  3((3+\tau) \sqrt{\frac{1+n}{T}} + \tau) = \eps$ for our choice of $T=36(3+\tau)^2(n+1)/\eps^2 = \poly(1/\eps, n)$ and $\tau=\eps/6$.

Two notes are in order. First, while the \textsc{Approx-Proj} algorithm of \cite{KakadeKL09} is described as a randomized algorithm and Theorem 3.2 is stated in terms of expected performance, this is only a cosmetic difference. This is because, instead of outputting a probability distribution over choices, they output a choice at random from this distribution and bound its expected performance, which is the same by linearity as the expected performance of the probability distribution. Hence, one can run their algorithm and simply keep track of the choice distribution. Second, they define $\calW_+=\{aw~|~a\geq 0, w\in \calW\}$ and require projection oracle onto $\calW_+$. It is not difficult to see that projection on to the set $\calW_+$ for our $\calW$, can be computed efficiently as required. 
\end{proof}

\end{document}